\documentclass{article}
\usepackage{ijcai13}

\usepackage{times}

\usepackage{mathrsfs}
\usepackage{dsfont}
\usepackage{amsmath}
\usepackage{amsthm}
\usepackage{amssymb}

\newtheorem{thm}{Theorem}
\newtheorem{cor}{Corollary}
\newtheorem{lem}{Lemma}
\newtheorem{prop}{Proposition}
\theoremstyle{definition}

\theoremstyle{remark}
\newtheorem{rem}{Remark}
\theoremstyle{example}

\title{Disjunctive Logic Programs versus Normal Logic Programs}

\author{Heng Zhang \and Yan Zhang \\
Artificial Intelligence Research Group\\
School of Computing, Engineering and Mathematics\\ University of Western Sydney, Australia}

\begin{document}

\maketitle

\begin{abstract}
This paper focuses on the expressive power of disjunctive and normal logic programs under the stable model semantics over finite, infinite, or arbitrary structures. A translation from disjunctive logic programs into normal logic programs is proposed and then proved to be sound over infinite structures. The equivalence of expressive power of two kinds of logic programs over arbitrary structures is shown to coincide with that over finite structures, and coincide with whether or not NP is closed under complement. Over finite structures, the intranslatability from disjunctive logic programs to normal logic programs is also proved if arities of auxiliary predicates and functions are bounded in a certain way.
\end{abstract}

\section{Introduction}

Normal logic programs provide us an elegant and efficient language for knowledge representation, which incorporates the abilities of classical logic, mathematical induction and nonmonotonic reasoning. Disjunctive logic programs extend this language by introducing epistemic disjunction to the rule head, motivated to represent more knowledge, in particular, indefinite knowledge. The most popular semantics for them is the stable model semantics, which was originally proposed by~\cite{GL:88}. Logic programming based on it is known as answer set programming, a flourishing paradigm of declarative programming emerged recently.

Identifying the expressive power of languages is one of the central topics in the area of knowledge representation and reasoning. In this paper we try to compare the expressive power of these two kinds of logic programs under the stable model semantics. The expressive power of logic programs as a database query language has been thoroughly studied in last three decades. For a survey, please refer to~\cite{DEGV:01}. However, except for a few work, the results for normal and disjunctive logic programs are limited to Herbrand structures. As encoding knowledge in Herbrand domains is unnatural and inflexible in many cases, our work will focus on infinite structures, finite structures and arbitrary structures. The semantics employed here is the general stable model semantics, which was developed by~\cite{FLL:11,LZ:11} via a second-order translation, and provides us a unified framework for answer set programming.

%
%
%





%

Our contributions in this paper are as follows. Firstly, we show that, over infinite structures, every disjunctive logic program can be equivalently translated to a normal logic program. Secondly, we prove that disjunctive and normal logic programs are of the same expressive power over arbitrary structures if and only if they are of the same expressive power over finite structures, and if and only if complexity class NP is closed under complement. Thirdly, we show that for each integer $k>1$ there is a disjunctive logic program with intensional predicates of arities $\leq k$ that can not be equivalently translated to any normal program with auxiliary predicates and functions of arities $<2k$. To prove them, the relationship between logic programs and classical logic is also studied.

\section{Preliminaries}

{\em Vocabularies} are assumed to be sets of predicate constants
and function constants. Every constant is
equipped with a natural number, its {\em arity}. Nullary function constants are also called {\em individual constants}, and nullary predicate constants are called {\em proposition constants}. For some technical reasons, a vocabulary is allowed to contain an arbitrary infinite set of proposition constants.
Logical symbols are defined as usual, including a countable set of
predicate variables, a countable set of function variables and a countable set of individual variables.
Predicate (function) constants and variables are simply called {\em predicates} ({\em functions})
if no confusion occurs.
Terms, formulae and sentences of a vocabulary $\upsilon$ (or shortly,
$\upsilon$-terms, $\upsilon$-formulae and $\upsilon$-sentences) are built
from $\upsilon$, equality, variables, connectives and quantifiers in a standard way.
For each formula $\varphi$ and each set $\Sigma$ of formulae, let $\upsilon(\varphi)$ and $\upsilon(\Sigma)$ be the sets of all constants occurring in $\varphi$ and $\Sigma$ respectively.
Let $Q\tau$ and $Q\bar{x}$ denote quantifier blocks $QX_1\cdots QX_n$ and $Qx_1\cdots Qx_m$ respectively if $\tau$ is the set of $X_i$ for all integers $1\leq i\leq n$, $\bar{x}=x_1\cdots x_m$, $Q$ is $\forall$ or $\exists$, $X_j$ and $x_i$ are predicate/function and individual variables respectively.


Every {\em structure} $\mathds{A}$ of a vocabulary $\upsilon$ (or
shortly, {\em $\upsilon$-structure} $\mathds{A}$) is accompanied by
a nonempty set $A$, the {\em domain} of $\mathds{A}$, and
interprets each $n$-ary predicate constant $P$ in $\upsilon$ as an
$n$-ary relation $P^{\mathds{A}}$ on $A$, and interprets each
$n$-ary function constant $f$ in $\upsilon$ as an $n$-ary function
$f^{\mathds{A}}$ on $A$. A structure is {\em finite} if its domain is finite; otherwise it is {\em infinite}.
A {\em restriction} of
a structure $\mathds{A}$ to a vocabulary $\sigma$ is the structure obtained
from $\mathds{A}$ by discarding all interpretations for constants
which do not belong to $\sigma$. Furthermore, given a vocabulary $\upsilon$, a structure
$\mathds{A}$ is called an {\em $\upsilon$-expansion} of some
$\sigma$-structure $\mathds{B}$ if $\sigma\subseteq\upsilon$, the vocabulary of $\mathds{A}$ is $\upsilon$, and
$\mathds{B}$ is a restriction of $\mathds{A}$ to $\sigma$.

Every {\em assignment} in a structure
$\mathds{A}$ is a function that maps each individual
variable to an element of $A$ and that maps each predicate
variable to a relation on $A$ of the same arity. Given a (second-order) formula
$\varphi$ and an assignment $\alpha$ in $\mathds{A}$, we write
$(\mathds{A},\alpha)\models\varphi$ if $\alpha$ {\em satisfies}
$\varphi$ in $\mathds{A}$ in the standard way. In particular, if
$\varphi$ is a sentence, we simply write $\mathds{A}\models\varphi$, and say
$\mathds{A}$ is a {\em model} of $\varphi$, or in other
words, $\mathds{A}$ {\em satisfies} $\varphi$. Given two (second-order) formulae $\varphi,\psi$ and a class $\mathcal{C}$ of structures, we say $\varphi$ is {\em equivalent} to $\psi$ over $\mathcal{C}$, or write $\varphi\equiv_{\mathcal{C}}\psi$ for short, if for every structure $\mathds{A}$ in $\mathcal{C}$ and every assignment $\alpha$ in $\mathds{A}$, $\alpha$ satisfies $\varphi$ in $\mathds{A}$ if and only if $\alpha$ satisfies $\psi$ in $\mathds{A}$.

Suppose $\tau$ is a set of predicates and $A$ is a domain, i.e. a nonempty set. Let $\textsc{ga}(\tau,A)$ denote the set of $P(\bar{a})$ for all predicates $P\in\tau$ and all $n$-tuples $\bar{a}$ on $A$ where $n$ is the arity of $P$. Let $\textsc{gpc}(\tau,A)$ be the set of finite disjunctions built from atoms in $\textsc{ga}(\tau,A)$. Each element in $\textsc{ga}(\tau,A)$ ($\textsc{gpc}(\tau,A)$) is called a {\em grounded atom} ({\em grounded positive clause}) of $\tau$ over $A$. Given a structure $\mathds{A}$, let $\textsc{Ins}(\mathds{A},\tau)$ be the set of grounded atoms $P(\bar{a})$ such that $P\in\tau$ and $P(\bar{a})$ is true in $\mathds{A}$.



Let $\mathsf{FIN}$ denote the class of all finite structures, and let $\mathsf{INF}$ denote the class of all infinite structures. Suppose $\Sigma$ and $\Pi$ are two sets of second-order formulae and let $\mathcal{C}$ be a class of structures. We write $\Sigma\leq_{\mathcal{C}}\Pi$ if for each formula $\varphi$ in $\Sigma$, there is a formula $\psi$ in $\Pi$ such that $\varphi\equiv_{\mathcal{C}}\psi$. We write $\Sigma\simeq_{\mathcal{C}}\Pi$ if both $\Sigma\leq_{\mathcal{C}}\Pi$ and $\Pi\leq_{\mathcal{C}}\Sigma$ hold. In particular, if $\mathcal{C}$ is the class of all arbitrary structures, the subscript $\mathcal{C}$ may be dropped.

\subsection{Logic Programs}

Every {\em disjunctive logic program} is a set of {\em rules} of the form\vspace{-.1cm}
\begin{equation*}\label{eqn:rule}
\zeta_1\wedge\cdots\wedge\zeta_{m}\rightarrow\zeta_{m+1}\vee\cdots\vee\zeta_{n}\vspace{-.1cm}
\end{equation*}
where $1\leq m\leq n$, and for each integer $m<i\leq n$, $\zeta_i$ is an atom without equality; for each integer $1\leq j\leq m$, $\zeta_j$ is a {\em literal}, i.e., an atom or its negation. The disjunctive part of the rule is called its {\em head}, and the conjunctive part called its {\em body}. Let $\Pi$ be a disjunctive logic program. Then each {\em intensional predicate} of $\Pi$ is a predicate constant that occurs in the head of some rule in $\Pi$. Atoms built from  intensional predicates of $\Pi$ are called {\em intensional atoms} of $\Pi$.

Let $\Pi$ be a disjunctive logic program. Then $\Pi$ is {\em normal} if the head of each rule contains at most one atom, $\Pi$ is {\em plain} if there is no negation of any intensional atom of $\Pi$ occurring in any of its rule,
$\Pi$ is {\em propositional} if no predicate of positive arity occurs in any of its rules, and $\Pi$ is {\em finite} if it contains only a finite set of rules. In particular, if we do not mention, a logic program is always assumed to be finite.

Given a disjunctive logic program $\Pi$, let $\mathrm{SM}(\Pi)$ denote the formula
$\varphi\wedge\forall\tau^{\ast}(\tau^{\ast}<\tau\rightarrow\neg\varphi^{\ast})$,
where $\tau$ is the set of all intensional predicate constants of $\Pi$; $\tau^{\ast}$ is the set of predicate variables $P^{\ast}$ for all $P$ in $\tau$; $\tau^{\ast}<\tau$ is the formula
$\wedge_{P\in\tau}\forall\bar{x}(P^{\ast}(\bar{x})\rightarrow P(\bar{x}))\wedge\neg\wedge_{P\in\tau}\forall\bar{x}(P(\bar{x})\rightarrow P^{\ast}(\bar{x}))$;
$\varphi$ is the conjunction of all sentences $\forall(\gamma)$ such that $\gamma$ is a rule in $\Pi$ and $\forall(\gamma)$ is the universal closure of $\gamma$; $\varphi^{\ast}$ is the conjunction of $\forall(\gamma^{\ast})$ such that $\gamma$ is a rule in $\Pi$ and $\gamma^{\ast}$ is the rule obtained from $\gamma$ by substituting $P^{\ast}(\bar{t})$ for all positive occurrences of $P(\bar{t})$ in its head or in its body if $P$ is in $\tau$. A structure $\mathds{A}$ is a {\em stable model} of $\Pi$ if it is a model of $\mathrm{SM}(\Pi)$.

Now, given a class $\mathcal{C}$ of structures, or in other words, a {\em property}, we can define it by a logic program in the following way: the models of second-order formula $\exists\tau\mathrm{SM}(\Pi)$ are exactly the structures in $\mathcal{C}$, where $\tau$ is a set of predicate and function constants occurring in $\Pi$. Constants in $\tau$ are called {\em auxiliary constants}. Given $n\geq 0$, let $\mathrm{DLP}^n$ ($\mathrm{DLP}^n_{\mathrm{F}}$) be the set of formulae $\exists\tau\mathrm{SM}(\Pi)$ for all disjunctive logic programs $\Pi$ and all finite sets $\tau$ of predicate (predicate and function, respectively) constants of arities $\leq n$. Let $\mathrm{DLP}$ ($\mathrm{DLP}_{\mathrm{F}}$) be the union of $\mathrm{DLP}^n$ ($\mathrm{DLP}^n_{\mathrm{F}}$, respectively) for all $n\geq 0$. In above definitions, if $\Pi$ is restricted to be normal, we then obtain the notations $\mathrm{NLP}^n,\mathrm{NLP}^n_{\mathrm{F}},\mathrm{NLP}$ and $\mathrm{NLP}_{\mathrm{F}}$ respectively.


Given a rule $\gamma$, a structure $\mathds{A}$ and an assignment $\alpha$ in $\mathds{A}$, let $\gamma[\alpha]$ be the rule obtained from $\gamma$ by substituting $P(\bar{a})$ for all atoms $P(\bar{t})$ where $\bar{a}=\alpha(\bar{t})$, let $\gamma^-_{\textsc{B}}$ be the set of all conjuncts in the body of $\gamma$ in which no intensional predicate positively occurs, and let $\gamma^+$ be the rule obtained from $\gamma$ by removing all literals in $\gamma^-_{\textsc{B}}$. Given a disjunctive logic program $\Pi$, let $\Pi^{\mathds{A}}$ be the set of rules $\gamma^{+}[\alpha]$ for all assignments $\alpha$ in $\mathds{A}$ and all rules $\gamma$ in $\Pi$ such that $\alpha$ satisfies $\gamma^-_{\textsc{b}}$ in $\mathds{A}$. The following proposition shows that the general stable model semantics can be redefined by the above first order GL-reduction:

\begin{prop}[\cite{ZZ:13}, Proposition 4]\label{prop:fo2prop}
Let $\Pi$ be a disjunctive logic program with a set $\tau$ of intensional predicates. Then an $\upsilon(\Pi)$-structure $\mathds{A}$ is a stable model of $\Pi$ iff $\textsc{Ins}(\mathds{A},\tau)$ is a minimal (via set inclusion) model of $\Pi^{\mathds{A}}$.
\end{prop}
%

\subsection{Progression Semantics}

In this paper, every clause and the clauses obtained from it by laws of commutation, association and identity for $\vee$ are regarded to be the same. Now we review a progression semantics proposed by~\cite{ZZ:13}, which generalizes the fixed point semantics of~\cite{LMR:92} to arbitrary structures and to logic programs with default negation.

%

Suppose $\Pi$ is a propositional, possibly infinite and plain disjunctive logic program, and $\Sigma$ is a set of finite disjunctions of atoms in $\upsilon(\Pi)$. We define $\Gamma_{\Pi}(\Sigma)$ as the set of all positive clauses $H\vee C_1\vee\cdots\vee C_k$ such that $k\geq 0$ and there are a rule $p_1\wedge\cdots\wedge p_k\rightarrow H$ in $\Pi$ and a sequence of positive clauses $C_1\vee p_1,\dots,C_k\vee p_k$ in $\Sigma$. It is easy to verify that $\Gamma_{\Pi}$ is a monotonic function on the sets of positive clauses of $\upsilon(\Pi)$.

Now, by the first-order GL-reduction defined above, a progressional operator for first-order logic programs is then defined. Let $\Pi$ be a disjunctive logic program and let $\mathds{A}$ be a structure of $\upsilon(\Pi)$. We define  $\Gamma^{\mathds{A}}_{\Pi}$ as the operator $\Gamma_{\Pi^{\mathds{A}}}$. Furthermore, define $\Gamma^{\mathds{A}}_{\Pi}\uparrow 0$ as the empty set, and define $\Gamma^{\mathds{A}}_{\Pi}\uparrow n$ as the union of $\Gamma^{\mathds{A}}_{\Pi}\uparrow n-1$ and $\Gamma^{\mathds{A}}_{\Pi}(\Gamma^{\mathds{A}}_{\Pi}\uparrow n-1)$ for all integers $n>0$. Finally, let $\Gamma^{\mathds{A}}_{\Pi}\uparrow\omega$ be the union of $\Gamma^{\mathds{A}}_{\Pi}\uparrow n$ for all integers $n\geq 0$. The following proposition provides us a progression semantics for disjuntive logic programs:

\begin{prop}[\cite{ZZ:13}, Theorem 1]\label{prop:fxp2sm}
Let $\Pi$ be a disjunctive logic program, $\tau$ the set of all intensional predicates of $\Pi$, and $\mathds{A}$ a structure of $\upsilon(\Pi)$. Then
$\mathds{A}$ is a stable model of $\Pi$ iff $\textsc{Ins}(\mathds{A},\tau)$ is a minimal model of $\Gamma^{\mathds{A}}_{\Pi}\uparrow\omega$.
\end{prop}

\begin{rem}
In above proposition, it is easy to see that, if $\Pi$ is normal, $\mathds{A}$ is a stable model of $\Pi$ iff $\textsc{Ins}(\mathds{A},\tau)=\Gamma^{\mathds{A}}_{\Pi}\uparrow\omega$.
\end{rem}

\section{Infinite Structures}

In this section, we propose a translation that turns each disjunctive logic program to an equivalent normal logic program over infinite structures. The main idea is to encode grounded positive clauses by elements in the intended domain. With the encoding, we then simulate the progression of given disjunctive logic program by the progression of a normal program.

Firstly, we show how to encode each clause by an element.
Let $A$ be an infinite set. Every {\em encoding function} on $A$ is an injective function from $A\times A$ into $A$. Let $\textsl{enc}$ be an encoding function on $A$ and $c$ an element in $A$ such that $\textsl{enc}(a,b)\neq c$ for all elements $a,b\in A$. For the sake of convenience, we let $\textsl{enc}(a_1,\dots,a_k;c)$ denote the following expression $$\textsl{enc}((\cdots(\textsl{enc}(c,a_1),a_2),\cdots),a_k)$$ for any integer $k\geq 0$ and any set of elements $a_1,\dots,a_k\in A$. Let $\textsl{enc}(A,c)$ denote the set $\{\textsl{enc}(\bar{a};c):\bar{a}\in A^{\ast}\}$ where $A^{\ast}$ is the set of all finite tuples of elements in $A$. The {\em merging function} $\textsl{mrg}$ on $A$ related to $\textsl{enc}$ and $c$ is the function from $\textsl{enc}(A,c)\times\textsl{enc}(A,c)$ into $\textsl{enc}(A,c)$ that satisfies
$$
\textsl{mrg}(\textsl{enc}(\bar{a};c),\textsl{enc}(\bar{b};c))=\textsl{enc}(\bar{a},\bar{b};c)
$$
for all tuples $\bar{a},\bar{b}\in A^{\ast}$. The {\em extracting function} $\textsl{ext}$ on $A$ related to $\textsl{enc}$ and $c$ is the function from $\textsl{enc}(A,c)\times A$ into $\textsl{enc}(A,c)$ that satisfies $\textsl{ext}(\textsl{enc}(\bar{a};c),b)=\textsl{enc}(\bar{a}';c)$, where $\bar{a}'$ is the tuple obtained from $\bar{a}$ by removing all occurrences of $b$. It is clear that both the merging function and the extracting function are unique if $\textsl{enc}$ and $c$ are fixed.

As mentioned before, the order of disjuncts in a clause does not change the semantics. To omit the order, we need some encoding predicates related to $\textsl{enc}$ and $c$. The predicate   $\textsl{in}$ is a subset of $\textsl{enc}(A,c)\times A$ such that $(\textsl{enc}(\bar{a},c),b)\in\textsl{in}$ iff $b$ occurs in $\bar{a}$; the predicate $\textsl{subc}$ is a binary relation on $\textsl{enc}(A,c)$ such that $(\textsl{enc}(\bar{a},c),\textsl{enc}(\bar{b},c))\in\textsl{subc}$ iff all the elements in $\bar{a}$ occur in $\bar{b}$. The predicate $\textsl{equ}$ is a binary relation on $\textsl{enc}(A,c)$ such that
$(a,b)\in\textsl{equ}$ iff $(a,b)\in\textsl{subc}$ and $(b,c)\in\textsl{subc}$.

\smallskip
\noindent{\bf Example 1.}
Let $\mathbb{Z}^+$ be the set of all positive integers, and define $e(m,n)=2^{m}+3^{n}$ for all integers $m,n\in\mathbb{Z}^+$. Then $e$ is clearly an encoding function on $\mathbb{Z}^+$, and integers $1,2,3,4$ are not in the range of $e$. Suppose $P_1,P_2,P_3$ are predicates. Then a grounded atom $P_2(1,3,5)$ can be encoded as $e(1,3,5;2)$, i.e. $e(e(e(2,1),3),5)$ that equals to $2^{155}+3^5$; the positive clause $P_2(1,3,5)\vee P_3(2)\vee P_1(2,4)$ can be encoded as $e(e(1,3,5;2),e(2;3),e(2,4;1);4)$, where, for each $1\leq i\leq 3$, integer $i$ is to be used for the ending flag of atoms $P_i(\cdots)$, and integer $4$ is for the ending flag of clauses.\hfill$\Box$\vspace{.1cm}

With this method for encoding, we can then define a translation.
Let $\Pi$ be a disjunctive logic program. We first construct a class of logic programs related to $\Pi$ as follows:

1. Let $C_{\Pi}$ denote the set consisting of an individual constant $c_{P}$ for each predicate constant $P$ that occurs in $\Pi$, and of an individual constant $c_{\epsilon}$. Let $\Pi_1$ consist of the rule
\begin{eqnarray}
\textsc{enc}(x,y,c)&\rightarrow&\bot
\end{eqnarray}
for each individual constant $c\in C_{\Pi}$, and the following rules:
\begin{eqnarray}
\!\!\!\!\!\!\neg\underline{\textsc{enc}}(x,y,z)\!\!\!&\rightarrow&\!\!\!\textsc{enc}(x,y,z)\\
\!\!\!\!\!\!\neg\textsc{enc}(x,y,z)\!\!\!&\rightarrow&\!\!\!\underline{\textsc{enc}}(x,y,z)\\
\!\!\!\!\!\!\textsc{enc}(x,y,z)\wedge\textsc{enc}(u,v,z)\wedge\neg x=u\!\!\!&\rightarrow&\!\!\!\bot\\
\!\!\!\!\!\!\textsc{enc}(x,y,z)\wedge\textsc{enc}(u,v,z)\wedge\neg y=v\!\!\!&\rightarrow&\!\!\!\bot
\end{eqnarray}
\begin{eqnarray}
\!\!\!\!\!\!\textsc{enc}(x,y,z)\!\!\!&\rightarrow&\!\!\!\textsc{ok}_e(x,y)\\
\!\!\!\!\!\!\neg\textsc{ok}_e(x,y)\!\!\!&\rightarrow&\!\!\!\textsc{ok}_e(x,y)\\
\!\!\!\!\!\!\textsc{enc}(x,y,z)\wedge\textsc{enc}(x,y,u)\wedge\neg z=u\!\!\!&\rightarrow&\!\!\!\bot
\end{eqnarray}\vspace{-.5cm}

2. Let $\Pi_2$ be the program consisting of the following rules:\vspace{-.15cm}
\begin{eqnarray}
\!\!\!\!\!\!\!y=c_{\epsilon}\!\!\!\!&\rightarrow&\!\!\!\!\textsc{mrg}(x,y,x)\\
\!\!\!\!\!\!\!\left[\begin{aligned}
\textsc{mrg}(x,u,v)&\wedge&\!\!\!\!\!\textsc{enc}(u,w,y)\\
&\wedge&\!\!\!\!\!\textsc{enc}(v,w,z)
\end{aligned}\right]
\!\!\!\!\!&\rightarrow&\!\!\!\!\textsc{mrg}(x,y,z)\\
\!\!\!\!\!\!\!x=c_{\epsilon}\!\!\!\!&\rightarrow&\!\!\!\!\textsc{ext}(x,y,x)\\
\!\!\!\!\!\!\!\textsc{ext}(u,y,v)\wedge\textsc{enc}(u,w,x)\wedge w=y\!\!\!\!&\rightarrow&\!\!\!\!\textsc{ext}(x,y,v)\\
\!\!\!\!\!\!\!\left[\begin{aligned}
\textsc{ext}(u,y,v)\wedge\textsc{enc}(u,w,x)\\
\wedge \neg w=y\wedge\textsc{enc}(v,w,z)
\end{aligned}\right]
\!\!\!\!\!&\rightarrow&\!\!\!\!\textsc{ext}(x,y,z)\\
\!\!\!\!\!\!\!\textsc{enc}(x,u,y)\!\!\!\!&\rightarrow&\!\!\!\!\textsc{in}(u,y)\\
\!\!\!\!\!\!\!\textsc{enc}(x,v,y)\wedge\textsc{in}(u,x)\!\!\!\!&\rightarrow&\!\!\!\!\textsc{in}(u,y)\\
\!\!\!\!\!\!\!x=c_{\epsilon}\!\!\!\!&\rightarrow&\!\!\!\!\textsc{subc}(x,y)\\
\!\!\!\!\!\!\!\textsc{subc}(u,y)\wedge\textsc{enc}(u,v,x)\wedge\textsc{in}(v,y)\!\!\!\!&\rightarrow&\!\!\!\!\textsc{subc}(x,y)\\
\!\!\!\!\!\!\!\textsc{subc}(x,y)\wedge\textsc{subc}(y,x)\!\!\!\!&\rightarrow&\!\!\!\!\textsc{equ}(x,y)
\end{eqnarray}\vspace{-.5cm}

\noindent

3. Let $\Pi_3$ be the logic program consisting of the rule\vspace{-.15cm}
\begin{eqnarray}
\label{eqn:simequ}
\textsc{true}(u)\wedge\textsc{equ}(u,v)&\rightarrow&\textsc{true}(v)
\end{eqnarray}\vspace{-.55cm}

\noindent
and the rule\vspace{-.15cm}
\begin{eqnarray}
\label{eqn:progsim}
\!\!\!\!\!\!\!\!\!\!\!\left[\begin{aligned}
\textsc{true}(x_1)\wedge z_1=\lceil\vartheta_1\rceil\wedge\textsc{in}(z_1,x_1)\wedge\cdots\\
\wedge\textsc{true}(x_k)\wedge z_k=\lceil\vartheta_k\rceil\wedge\textsc{in}(z_k,x_k)\\
\wedge\textsc{ext}(x_1,z_1,y_1)\wedge\cdots\wedge\textsc{ext}(x_k,z_k,y_k)\\
\wedge\textsc{mrg}(y_1,\dots,y_k,\lceil\gamma_{\textsc{h}}\rceil,v)\wedge\gamma^-_{\textsc{b}}\end{aligned}\right]
\!\!\!\!\!\!\!&\rightarrow&\!\!\!\!\!\textsc{true}(v)
\end{eqnarray}\vspace{-.25cm}

\noindent
for each rule $\gamma$ in $\Pi$, where $\vartheta_1,\dots,\vartheta_k$ list all the atoms that have strictly positive occurrences in the body of $\gamma$ for some integer $k\geq 0$; $\gamma_{\textsc{h}}$ is the head of $\gamma$, $\gamma^-_{\textsc{b}}$ is the conjunction of literals occurring in the body of $\gamma$ but not in $\vartheta_1,\dots,\vartheta_k$; for each integer $1\leq i\leq k$, $z_i=\lceil\vartheta_i\rceil$ denotes the formula\vspace{-.15cm}
\begin{align*}
\textsc{enc}(c_P,t_1,u^i_1)&\wedge\textsc{enc}(u^i_1,t_2,u^i_2)\wedge\cdots\\
&\wedge\textsc{enc}(u^i_{m-1},t_m,u^i_m)\wedge z_i=u^i_m
\end{align*}\vspace{-.55cm}

\noindent
if $\vartheta_i=P(t_1,\dots,t_m)$; $\textsc{mrg}(y_1,\dots,y_k,\lceil\gamma_{\textsc{h}}\rceil,v)$ denotes\vspace{-.15cm}
\begin{align*}
\textsc{enc}(c_{\epsilon},&\lceil\zeta_1\rceil,v_1)\wedge\textsc{enc}(v_1,\lceil\zeta_2\rceil,v_2)\wedge\cdots\\
&\wedge\textsc{enc}(v_{n-1},\lceil\zeta_n\rceil,v_n)\wedge\textsc{mrg}(y_1,y_2,w_2)\\
&\wedge\cdots\wedge\textsc{mrg}(w_{k-1},y_k,w_k)\wedge\textsc{enc}(w_k,v_n,v)
\end{align*}\vspace{-.55cm}

\noindent
if $\gamma_{\textsc{h}}=\zeta_1\vee\cdots\vee\zeta_n$ for some atoms $\zeta_1,\dots,\zeta_n$ and $n\geq 0$.

4. Let $\Pi_4$ be the program consisting of the rule\vspace{-.15cm}
\begin{eqnarray}
\label{eqn:pi5_false1}x=c_{\epsilon}&\rightarrow&\textsc{false}(x)
\end{eqnarray}\vspace{-.55cm}

\noindent
and the rule\vspace{-.15cm}
\begin{eqnarray}
\label{eqn:pi5_false2}\!\!\!\!\!\!\!\textsc{false}(x)\wedge\textsc{enc}(x,\lceil\vartheta\rceil,y)\wedge\neg\vartheta\!\!\! &\rightarrow&\!\!\! \textsc{false}(y)
\end{eqnarray}\vspace{-.55cm}

\noindent for every intensional predicate $P$ of $\Pi$ and every atom $\vartheta$ of the form $P(\bar{z}_P)$, where $z_1,z_2,\dots$ are individual variables, $k_P$ is the arity of $P$, and $\bar{z}_P$ denotes the tuple $z_1\cdots z_{k_P}$.

5. Let $\Pi_5$ be the logic program consisting of the rule\vspace{-.15cm}
\begin{eqnarray}
\label{eqn:pi5_true1}\textsc{true}(c_{\epsilon})&\rightarrow&\bot
\end{eqnarray}\vspace{-.55cm}

\noindent and the following rule\vspace{-.15cm}
\begin{eqnarray}
\label{eqn:pi5_true2}\!\!\!\!\!\!\!\textsc{true}(x)\wedge\textsc{ext}(x,\lceil\vartheta\rceil,y)\wedge\textsc{false}(y)\!\!\! &\rightarrow&\!\!\!\vartheta
\end{eqnarray}\vspace{-.55cm}

\noindent for each atom $\vartheta$ of the form same as that in $\Pi_4$.

In the end, define $\Pi^{\diamond}$ as the union of programs $\Pi_1,\dots,\Pi_5$.

Next, we explain the intuition of this translation. Program $\Pi_1$ assures that $\textsc{enc}$ will be interpreted as an encoding function on the domain and constants in $C_{\Pi}$ will be interpreted as ending flags for the encoding of atoms and positive clauses. Program $\Pi_2$ defines the merging and extracting functions, and all the encoding predicates mentioned before. Based on these assumptions, program $\Pi_3$ then simulates the progression of $\Pi$ on encodings of positive clauses. As all the needed clauses will be derived in $\omega$ stages, this simulation is realizable. Lastly, we use program $\Pi_5$ and $\Pi_4$ to decode the encodings of positive clauses. The only difficulty in this decoding is that we need represent the resulting positive clauses by a program without disjunction. As the set of positive clauses are clearly head-cycle-free, we can get over it by applying the shift operation presented in~\cite{BD:94}.

\begin{thm}
Let $\Pi$ be a disjunctive logic program. Then over infinite structures, $\textsc{SM}(\Pi)$ is equivalent to $\exists\varsigma\textsc{SM}(\Pi^{\diamond})$, where $\varsigma$ denotes the set of constants occurring in $\Pi^{\diamond}$ but not in $\Pi$.
\end{thm}

\begin{proof}
Let $\upsilon_1,\dots,\upsilon_5$ and $\tau$ be the set of all intensional predicates of $\Pi_1,\dots,\Pi_5$ and $\Pi$ respectively. Let $\sigma$ be the union of $\upsilon_1,\upsilon_2$ and $\upsilon(\Pi)$. Then by Splitting Lemma in~\cite{FLLP:09} and the second-order transformation, we can conclude\vspace{-.1cm}
\begin{equation}\label{eqn:splitting}
\mathrm{SM}(\Pi^{\diamond})\equiv\mathrm{SM}(\Pi_1)\wedge\cdots\wedge\mathrm{SM}(\Pi_5).\vspace{-.1cm}
\end{equation}

Let $\mathds{A}$ be any infinite structure of $\upsilon(\Pi)$, and let $\mathds{B}$ be any $\sigma$-expansion of $\mathds{A}$ that satisfies the following conditions:
\begin{enumerate}
\item $\textsc{enc}$ is interpreted as the graph of an encoding function $\textsl{enc}$ on $A$ such that no element among $c_{\epsilon}^{\mathds{B}}$ and $c_P^{\mathds{B}}$ (for all $P\in\tau$) belongs to the range of $\textsl{enc}$, $\underline{\textsc{enc}}$ is interpreted as the complement of the graph of $\textsl{enc}$, $\textsc{ok}_e^{\mathds{B}}=A\times A$;
\item $\textsc{mrg}$ and $\textsc{ext}$ are interpreted as graphs of the merging and extracting functions related to $\textsl{enc}$ and $c_{\epsilon}^{\mathds{B}}$ respectively, $\textsc{in},\textsc{subc},\textsc{equ}$ are interpreted as encoding predicates $\textsl{in},\textsl{subc},\textsl{equ}$ related to $\textsl{enc}$ and $c_{\epsilon}^{\mathds{B}}$ respectively.
\end{enumerate}

For convenience, we need define some notations. Given a grounded atom $\vartheta$ of the form $P(a_1,\dots,a_k)$ for any $P\in\tau$, let $\langle\vartheta\rangle$ be short for $\textsl{enc}(a_1,\dots,a_k;c_P^{\mathds{B}})$. Given a grounded clause $C$ in $\textsc{gpc}(\tau,A)$ of the form $\vartheta_1\vee\cdots\vee\vartheta_n$ where each $\vartheta_i$ is an atom, let $\langle C\rangle$ denote $\textsl{enc}(\langle\vartheta_1\rangle,\dots,\langle\vartheta_n\rangle;c_{\epsilon}^{\mathds{B}})$. Given a set $\Sigma\subseteq\textsc{gpc}(\tau,A)$, let $\langle\Sigma\rangle$ denote the set $\{\langle C\rangle:C\in\Sigma\}$. Moreover, let $\Delta^{n}(\mathds{B})=\{a\in B:\textsc{true}(a)\in\Gamma^{\mathds{B}}_{\Pi_3}\uparrow n\}$. By an induction on $n$, we can show the following claim:

\smallskip
\noindent{\em Claim} 1. For all integers $n\geq 0$, $\langle\Gamma^{\mathds{A}}_{\Pi}\uparrow n\rangle=\Delta^n(\mathds{B})$.
\smallskip

Now, let $\mathds{B}^+$ be the $\upsilon(\Pi^{\diamond})$-expansion of $\mathds{B}$ that interprets $\textsc{true}$ as the set $\cup_{n\geq 0}\Delta^n(\mathds{B})$, and interprets $\textsc{false}$ as the set of $\langle C\rangle$ for all $C\in\textsc{gpc}(\tau,A)$ such that $\textsc{Ins}(\mathds{A},\tau)\models\neg C$.

%

\smallskip
\noindent{\em Claim} 2. Let $\mathds{C}=\mathds{B}^+$. Then $\textsc{Ins}(\mathds{A},\tau)$ is a minimal model of $\Gamma^{\mathds{A}}_{\Pi}\uparrow\omega$ iff $\textsc{Ins}(\mathds{C},\tau)$ is a minimal model of $\Pi_5^{\mathds{C}}$.
\smallskip

To prove this claim, we need check the validity of decoding, and then show that the shift operation preserves the semantics, which is similar to that of Theorem 4.17 in~\cite{BD:94}. Due to the limit of space, we leave the detailed proof to a full version of this paper.

\smallskip

With these two claims, we can prove the theorem now:

``$\Longrightarrow$": Suppose $\mathds{A}$ is an infinite model of $\mathrm{SM}(\Pi)$. Let $\mathds{B}$ be a $\sigma$-expansion of $\mathds{A}$ defined by conditions 1 and 2. The existence of expansion $\mathds{B}$ is clearly assured by the infiniteness of $A$. It is easy to check that $\mathds{B}$ is a stable model of both $\Pi_1$ and $\Pi_2$. Let $\mathds{C}$ be the structure $\mathds{B}^+$. Then it is also clear that $\mathds{C}$ is a stable model of both $\Pi_3$ and $\Pi_4$. On the other hand, by Proposition \ref{prop:fxp2sm} and the assumption, $\textsc{Ins}(\mathds{A},\tau)$ should be a minimal model of $\Gamma^{\mathds{A}}_{\Pi}\uparrow\omega$. So, by Claim 2, $\textsc{Ins}(\mathds{C},\tau)$ is a minimal model of $\Pi_5^{\mathds{C}}$, which means that $\mathds{C}$ is a stable model of $\Pi_5$ by Proposition \ref{prop:fo2prop}. By equation (\ref{eqn:splitting}), $\mathds{C}$ is then a stable model of $\Pi^{\diamond}$, which means that $\mathds{A}$ satisfies $\exists\varsigma\mathrm{SM}(\Pi^{\diamond})$.

``$\Longleftarrow$": Suppose $\mathds{A}$ is an infinite model of $\exists\varsigma\mathrm{SM}(\Pi^{\diamond})$. Then there is an $\upsilon(\Pi^{\diamond})$-expansion $\mathds{C}$ of $\mathds{A}$ such that $\mathds{C}$ is a stable model of $\Pi^{\diamond}$. Let $\mathds{B}$ be the restrictions of $\mathds{C}$ to $\sigma$. By equation (\ref{eqn:splitting}), $\mathds{B}$ is a model of the formulae $\mathrm{SM}(\Pi_1)$ and $\mathrm{SM}(\Pi_2)$, which implies that $\textsc{enc}$ is interpreted as an encoding function $\textsl{enc}$ on $A$, $c_{\epsilon}$ and $c_P$ for all predicates $P\in\tau$ are interpreted as elements not in the range of $\textsl{enc}$, predicates $\textsc{mrg},\textsc{ext},\textsc{in},\textsc{subc},\textsc{equ}$ are interpreted as the corresponding functions or predicates $\textsl{mrg},\textsl{ext},\textsl{in},\textsl{subc},\textsl{equ}$ related to $\textsl{enc}$ and $c_{\epsilon}^{\mathds{C}}$. By equation (\ref{eqn:splitting}) again, $\mathds{C}$ is a stable model of $\Pi_3$ and $\Pi_4$, and by Proposition \ref{prop:fxp2sm}, $\textsc{Ins}(\mathds{C},\upsilon_3)$ is then a minimal model of $\Gamma^{\mathds{B}}_{\Pi_3}\uparrow\omega$. So, we have $\mathds{C}=\mathds{B}^+$. According to Proposition \ref{prop:fo2prop}, $\textsc{Ins}(\mathds{C},\tau)$ should be a minimal model of $\Pi_5^{\mathds{C}}$ as $\mathds{C}$ is clearly a stable model of $\Pi_5$. Applying Claim 2, we then have that $\textsc{Ins}(\mathds{A},\tau)$ is a minimal model of $\Gamma^{\mathds{A}}_{\Pi}\uparrow\omega$. Thus, by Proposition \ref{prop:fxp2sm}, $\mathds{A}$ should be a stable model of $\Pi$.
\end{proof}

\begin{rem}
Note that, given any finite domain $A$, there is no injective function from $A\times A$ into $A$. Therefore, we can not expect that the above translation works on finite structures.
\end{rem}

From the above theorem, we then get the following result:

\begin{cor}
$\mathrm{DLP}\simeq_{\mathsf{INF}}\mathrm{NLP}$.
\end{cor}

\section{Finite Structures}

In this section, we will focus on the relationship between disjunctive and normal logic programs over finite structures.  In the general case, the separation of their expressive power is turned out to be very difficult due to the following result\footnote{A similar result over function-free Herbrand structures follows from the expressive power of traditional answer set programs. Here is a reformulation of it under the general stable model semantics.}:

\begin{prop}\label{prop:fin_dlp2nlp}
$\mathrm{DLP}\simeq_{\mathsf{FIN}}\mathrm{NLP}$ iff $\mathrm{NP}=\mathrm{coNP}$.
\end{prop}

\begin{proof}
Let $\Sigma^1_2$ denote the set of sentences of the form $\exists\tau\forall\sigma\varphi$, where $\tau$ and $\sigma$ are finite sets of predicate variables, $\varphi$ is a first-order formula. Let $\mathrm{ESO}$ be the set of sentences of the above form such that $\sigma$ is empty.
By Fagin Theorem~\shortcite{Fagin:74} and Stockmeyer's characterization of the polynomial hierarchy~\shortcite{Stoc:77}, we have that $\Sigma^1_2\simeq_{\mathsf{FIN}}\mathrm{ESO}$ iff $\Sigma^p_2=\mathrm{NP}$. By a routine complexity theoretical argument, we also have that $\Sigma^p_2=\mathrm{NP}$ iff $\mathrm{NP}=\mathrm{coNP}$. On the other hand, according to the proof of Theorem 6.3 in~\cite{EGM:97}, or by Lemma \ref{lem:so2dlp} in this section, Leivant's Normal Form~\shortcite{Leiv:89} and the definition of $\mathrm{SM}$, we can conclude that $\mathrm{DLP}\simeq_{\mathsf{FIN}}\Sigma^1_2$; by Lemma \ref{lem:nlp2eso} in this section\footnote{Note that functions in both existential second-order logic and normal logic programs can be easily simulated by predicates.}, it holds that $\mathrm{NLP}\simeq_{\mathsf{FIN}}\mathrm{ESO}$. Combining these results, we then have the desired proposition.
\end{proof}

Due to the significant difficulty of general separation, the rest of this section is devoted to a weaker separation between disjunctive and normal programs. To do this, we first study some relationship between logic programs and classical logic. In the following, let $\mathrm{ESO}^k_{\mathrm{F}}[\forall^{\ast}]$ denote the set of all sentences of the form $\exists\tau\forall\bar{x}\varphi$, where $\tau$ is a finite set of predicate and function variables of arity $\leq k$, and $\varphi$ is quantifier-free.


\begin{lem}\label{lem:nlp2eso}
$\mathrm{NLP}^k_{\mathrm{F}}\simeq_{\mathsf{FIN}}\mathrm{ESO}^k_{\mathrm{F}}[\forall^{\ast}]$ for all $k>1$.
\end{lem}

\begin{proof}
``$\geq_{\mathsf{FIN}}$": Let $\varphi$ be a sentence in $\mathrm{ESO}^k_{\mathrm{F}}[\forall^{\ast}]$. It is clear that $\varphi$ can be rewritten as an equivalent sentence of the form $\exists\tau\forall\bar{x}(\gamma_1\wedge\cdots\wedge\gamma_n)$ for some $n\geq 0$, where each $\gamma_i$ is a disjunction of atoms or negated atoms, and $\sigma$ a finite set of functions or predicates of arity $\leq k$. Let $\Pi$ be a logic program consisting of the rule $\tilde{\gamma}_i\rightarrow\bot$ for each $1\leq i\leq n$, where $\tilde{\gamma}_i$ is obtained from $\gamma_i$ by substituting $\vartheta$ for each negated atom $\neg\vartheta$, substituting $\neg\vartheta$ for each atom $\vartheta$, and substituting $\wedge$ for $\vee$. Obviously, $\exists\sigma\mathrm{SM}(\Pi)$ is in $\mathrm{NLP}^k_{\mathrm{F}}$ and equivalent to $\varphi$.

``$\leq_{\mathsf{FIN}}$": Let $\exists\sigma\mathrm{SM}(\Pi)$ be a formula in $\mathrm{NLP}^k_{\mathrm{F}}$ such that $\Pi$ is a normal logic program. Without loss of generality, we assume the head of each rule in $\Pi$ is of the form $P(\bar{x})$ for some integer $l\geq 0$ and $l$-ary intensional predicate $P$ of $\Pi$. Let $\tau$ and $S$ be the sets of all intensional predicates and atoms of $\Pi$ respectively. Let $c=k\cdot|\tau|+1$. For each $\lambda\in S$, suppose $\gamma_1,\dots,\gamma_{n}$ list all rules in $\Pi$ whose heads are $\lambda$. Suppose\vspace{-.12cm}
\begin{equation*}
\gamma_i=\zeta^i\wedge\vartheta^i_1\wedge\cdots\wedge\vartheta^i_{m_i}\rightarrow\lambda\vspace{-.12cm}
\end{equation*}
where $\vartheta^i_1,\dots,\vartheta^i_{m_i}$ are intensional atoms, $\zeta^i$ is a conjunction of literals that are not intensional atoms of $\Pi$, $m_i\geq 0$, and $\bar{y}_i$ is the tuple of all individual variables occurring in $\gamma_i$ but not in $\lambda$.
Now, we then define $\varphi_{\lambda}$ as the conjunction of formulae\vspace{-.12cm}
\begin{equation*}
(\zeta^i\wedge\vartheta^i_1\wedge\cdots\wedge\vartheta^i_{m_i})\vee\lambda\rightarrow\textsc{drvbl}(\lambda)\vspace{-.12cm}
\end{equation*}
for all $1\leq i\leq n$, and define $\psi_{\lambda}$ as the formula\vspace{-.12cm}
\begin{equation*}
\textsc{drvbl}(\lambda)\rightarrow\lambda\wedge\bigvee_{i=1}^{n}\exists\bar{y}_i\left[\zeta^i\wedge\bigwedge_{j=1}^{m_i}\textsc{drvless}(\vartheta^i_j,\lambda)\right]\vspace{-.12cm}
\end{equation*}
where $<$ is a new binary predicate and $\max$ a new individual constant, which are intended to be interpreted as a strict total order and the maximal element in the order respectively; $\bar{s}<\bar{t}$ is the formula describing that $\bar{s}$ is strictly less than $\bar{t}$ in the lexicographic order generated by $<$ if $\bar{s}$ and $\bar{t}$ are two tuples of terms of the same length; $\overline{\max}$ denotes the tuple $(\max,\dots,\max)$ of length $c$; $o_P^i$ is a new function constant of the same arity as $P$ if $1\leq i\leq c$ and $P$ is an intensional predicate of $\Pi$; $\mathrm{ord}(\vartheta)$ is the tuple $(o_P^c(\bar{t}),\cdots,o_P^{1}(\bar{t}))$ if $\vartheta$ is an intensional atom of form $P(\bar{t})$; $\textsc{drvbl}(\vartheta)$ denotes formula $\mathrm{ord}(\vartheta)<\overline{\max}$ for all atoms $\vartheta$; $\textsc{drvless}(\vartheta_1,\vartheta_2)$ denotes formula $\mathrm{ord}(\vartheta_1)<\mathrm{ord}(\vartheta_2)$ for all atoms $\vartheta_1$ and $\vartheta_2$.

Define $\varphi_{\Pi}$ as the universal closure of conjunction of the formula $\varpi$ and formulae $\varphi_{\lambda}\wedge\psi_{\lambda}$ for all $\lambda\in S$, where $\varpi$ is a sentence in $\mathrm{ESO}^2_{\mathrm{F}}[\forall^{\ast}]$ which describes that $<$ is a strict total order and $\max$ is the largest element in that order. By a moment's thought, we can construct such a sentence. Clearly, $\varphi_{\Pi}$ is equivalent to a sentence in $\mathrm{ESO}^k_{\mathrm{F}}[\forall^{\ast}]$ by introducing Skolem functions. Now we show that $\mathrm{SM}(\Pi)\equiv_{\mathsf{FIN}}\exists\sigma\varphi_{\Pi}$, where $\sigma$ is the set of constants occurring in $\varphi_{\Pi}$ but not in $\Pi$.

Due to the limit of space, we only show ``$\Longleftarrow$". Assume $\mathds{B}$ is a finite model of $\varphi_{\Pi}$. By the formula $\varpi$, the predicate $<$ is interpreted as a strict total order on $B$ and $\max^{\mathds{B}}$ is the maximal element with respect to the order. Let $\mathds{A}$ be the restriction of $\mathds{B}$ to the vocabulary $\upsilon(\Pi)$. To show that $\mathds{A}$ is a stable model of $\Pi$, by Proposition \ref{prop:fxp2sm} it suffices to show that $\textsc{Ins}(\mathds{A},\tau)=\Gamma^{\mathds{A}}_{\Pi}\uparrow\omega$. Through a routine induction on $n$, we can show that $\Gamma^{\mathds{A}}_{\Pi}\uparrow\omega\subseteq\textsc{Ins}(\mathds{A},\tau)$ for all integers $n\geq 0$. For all $P,Q$ in $\tau$ and all tuples $\bar{a},\bar{b}$ on $A$ of lengths corresponding to arities of $P,Q$ respectively, we define $P_{\bar{a}}\prec Q_{\bar{b}}$ iff $\mathds{B}$ satisfies $\textsc{drvless}(P(\bar{a}),Q(\bar{b}))$. Let $\vartheta\in\textsc{Ins}(\mathds{A},\tau)$, and let $\mathrm{rank}(\vartheta)$ be the number of intensional grounded atoms $\zeta$ over $A$ such that $\zeta\prec\vartheta$. By an induction on $\mathrm{rank}(\vartheta)$, we can show $\vartheta\in\Gamma^{\mathds{A}}_{\Pi}\uparrow\omega$.
Therefore, $\textsc{Ins}(\mathds{A},\tau)=\Gamma^{\mathds{A}}_{\Pi}\uparrow\omega$.
\end{proof}

By notation $\Sigma^1_{2,k}[\forall^k\exists^{\ast}]$ we denote the set of all sentences of the form $\exists\tau\forall\sigma\forall\bar{x}\exists\bar{y}\psi$, where $\tau$ and $\sigma$ are two finite sets of predicate variables of arity $\leq k$, $\bar{x}$ is a $k$-tuple of individual variables, $\psi$ is quantifier-free. Now we then have:

\begin{lem}\label{lem:so2dlp}
$\Sigma^{1}_{2,k}[\forall^k\exists^{\ast}]\leq_{\mathsf{FIN}}\mathrm{DLP}^k$ for all $k>1$.
\end{lem}

\begin{proof} (Sketch)
Let $\exists\tau\forall\sigma\varphi$ be any sentence in $\Sigma^{1}_{2,k}[\forall^k\exists^{\ast}]$ where $\tau,\sigma$ are finite sets of predicate variables of arities $\leq k$. Without loss of generality, suppose $\varphi=\forall\bar{x}\exists\bar{y}\vartheta(\bar{x},\bar{y})$, where $\bar{x}$ is of length $k$; $\vartheta$ is a formula built from literals and connectives $\wedge$ and $\vee$. By a modification of the program in the proof of Theorem 6.3 in~\cite{EGM:97}, we can obtain a disjunctive logic program with auxiliary predicates of arities $\leq k$ which defines the property expressed by $\exists\tau\forall\sigma\varphi$.

Let $\Pi_1$ be a logic program same as $\pi_1$ in the proof of Theorem 6.3 in~\cite{EGM:97} where the binary predicate $S$ defines a successor relation on the domain; the unary predicates $F$ and $L$ define the singleton sets that contain the least and the maximal elements in the order defined by $S$ respectively. Let $\Pi_2$ be the program consisting of (i) the rules $
X(\bar{u})\vee X^c(\bar{u})
$ for all predicate variables $X$ in $\tau$ or $\sigma$, (ii) the rules\vspace{-.12cm}
\begin{equation}
L(\bar{x})\wedge D(\bar{x})\rightarrow Y(\bar{u})\wedge Y^c(\bar{u})\vspace{-.12cm}
\end{equation}
for all predicate variables $Y$ in $\sigma$, (iii) the rules\vspace{-.12cm}
\begin{eqnarray}
\label{rule:so2dlp_u1}F(\bar{x})\wedge\vartheta^c(\bar{x},\bar{y})\rightarrow D(\bar{x})\\
\label{rule:so2dlp_u2}\bar{S}(\bar{z},\bar{x})\wedge D(\bar{z})\wedge\vartheta^c(\bar{x},\bar{y})\rightarrow D(\bar{x})
\end{eqnarray}\vspace{-.55cm}

\noindent
and (iv) the rule
$L(\bar{z})\wedge
\neg D(\bar{z})\rightarrow D(\bar{z})
$, where, for every predicate $X$ in $\tau$ or $\sigma$, $X^c$ is a new predicate of the same arity; $\vartheta^c$ is the formula obtained from $\vartheta$ by substituting $Y^c(\bar{t})$ for all occurrences of $\neg Y(\bar{t})$ if  $Y\in\sigma$ and $\bar{t}$ is a tuple of items; $D$ is a $k$-ary new predicate; $P(x_1,\dots,x_k)$ denotes the formula $P(x_1)\wedge\cdots\wedge P(x_k)$ if $P$ is $F$ or $L$; and $\bar{S}(\bar{x},\bar{z})$ denotes a quantifier-free formula defining the lexicographic order generated by $S$ on $k$-tuples. Note that, in general, formulae (\ref{rule:so2dlp_u1}) and (\ref{rule:so2dlp_u2}) are not rules defined previously. However, by applying some distributivity-like laws~\cite{CPV:05}, they can be replaced by an (strongly) equivalent set of rules.

Let $\Pi$ be the union of $\Pi_1$ and $\Pi_2$. By a similar but slightly more complex argument than that in Theorem 6.3 of~\cite{EGM:97}, we can show that $\forall\sigma\varphi\equiv_{\mathsf{FIN}}\exists\varsigma\mathrm{SM}(\Pi)$, where $\varsigma$ is the set of all predicates in $\upsilon(\Pi)$ but not in $\upsilon(\forall\sigma\varphi)$.
\end{proof}

With these two lemmas, we then have the following result:

\begin{thm}
$\mathrm{DLP}^{k}\not\leq_{\mathsf{FIN}}\mathrm{NLP}^{2k-1}_{\mathrm{F}}$ for all $k>1$.
\end{thm}

\begin{proof}
(Sketch) Let $n$ be an integer $\geq 1$ and $\upsilon$ the vocabulary consisting of only an $n$-ary predicate $P$. Define $\textsc{Parity}^n$ to be the class of $\upsilon$-structures in each of which $P$ is interpreted as a set consisting of an even number of $n$-tuples. We first show that the property $\textsc{Parity}^{2k}$ is definable in $\mathrm{DLP}^k$.

Let $\varphi_1$ be a formula asserting ``$S$ is interpreted as a successor relation on the domain; $0$ and $m$ are interpreted as the least and the maximal elements in the order defined by $S$ respectively".
Let $\varphi_2$ be the following formula:\vspace{-.15cm}
\begin{equation*}
\begin{aligned}
(Y(\bar{0})\leftrightarrow\,&P(\bar{x},\bar{0}))\wedge\forall\bar{u}\bar{v}[\bar{S}(\bar{u},\bar{v})\rightarrow
(P(\bar{x},\bar{v})\\&\leftrightarrow Y(\bar{v})\oplus Y(\bar{u}))]\rightarrow(X(\bar{x})\leftrightarrow Y(\bar{m}))
\end{aligned}\vspace{-.15cm}
\end{equation*}
where $\psi\oplus\chi$ denotes the formula $\psi\leftrightarrow\neg\chi$; $\bar{c}$ denotes the tuple $(c,\dots,c)$ of length $k$ if $c$ is $0$ or $m$. It is easy to see that $\varphi_2$ describes the property ``$X(\bar{a})$ is true iff the cardinality of the set $\{\bar{b}:P(\bar{a},\bar{b})\}$ is odd". Let $\varphi_3$ be the following formula:\vspace{-.15cm}
\begin{equation*}
\begin{aligned}
(X(\bar{0})\leftrightarrow\,&Y(\bar{0}))\wedge\forall\bar{u}\bar{v}[\bar{S}(\bar{u},\bar{v})\rightarrow\\
&(X(\bar{v})\leftrightarrow Y(\bar{v})\oplus Y(\bar{u}))]\rightarrow\neg Y(\bar{m})
\end{aligned}\vspace{-.15cm}
\end{equation*}
This formula asserts ``$X$ consists of an even number of $k$-tuples on the domain". Now let $\varphi$ be the following sentence:\vspace{-.15cm}
\begin{equation*}
\exists m\exists 0\exists S[\varphi_1\wedge\exists X\forall Y\forall\bar{x}(\varphi_2\wedge\varphi_3)]\vspace{-.15cm}
\end{equation*}
It is not difficult to check that $\varphi$ defines $\textsc{Parity}^{2k}$ over finite structures.
By Lemma \ref{lem:so2dlp}, there is a logic program $\Pi_{\mathrm{p}}$ such that $\exists X\forall Y\forall\bar{x}(\varphi_2\wedge\varphi_3)\equiv_{\mathsf{FIN}}\exists\tau\mathrm{SM}(\Pi_{\mathrm{p}})\in\mathrm{DLP}^k$, where $\tau$ is a finite set of predicates of arity $\leq k$. On the other hand, $\varphi_1$ can be easily encoded by a disjunctive logic program $\Pi_{\mathrm{o}}$\footnote{For example, such a program can be obtained by a slightly modification of $\pi_1$ in the proof of Lemma 6.4 in~\cite{EGM:97}.} involving only predicates of arity $\leq 2$. Hence, we have that\vspace{-.15cm}
$$\varphi\equiv_{\mathsf{FIN}}\exists m\exists 0\exists S[\exists\sigma\mathrm{SM}(\Pi_{\mathrm{o}})\wedge\exists\tau\mathrm{SM}(\Pi_{\mathrm{p}})].\vspace{-.15cm}$$
Let $\Pi$ be $\Pi_{\mathrm{o}}\cup\Pi_{\mathrm{p}}$. By Splitting Lemma in~\cite{FLLP:09},  $\varphi\equiv_{\mathsf{FIN}}\exists0\exists m\exists S\exists\sigma\exists\tau\mathrm{SM}(\Pi)$. Let $\Pi'$ be the program obtained from $\Pi$ by simulating individual constants $0$ and $m$ by unary predicates. This is then the desired program.

Next, we show that $\textsc{Parity}^{2k}$ is undefinable in $\mathrm{NLP}^{2k-1}_{\mathrm{F}}$ over finite structures.
By Lemma \ref{lem:nlp2eso}, it suffices to show that $\textsc{Parity}^{2k}$ is undefinable in $\mathrm{ESO}^{2k-1}_{\mathrm{F}}$. Towards a contradiction, assume that there is a sentence $\psi$ in this class such that finite models of $\psi$ are exactly the structures in $\textsc{Parity}^{2k}$. By employing an idea similar to that in Theorem 3.1 of~\cite{DLS:98}, we can then construct a formula $\psi_0$ in $\mathrm{ESO}^{4k-2}_{\mathrm{F}}$ to define $\textsc{Parity}^{4k}$. However, according to Theorem 2.1 of~\cite{Ajtai:83}, this is impossible since $\psi_0$ has an equivalent in $\mathrm{ESO}^{4k-1}$, i.e., the set of formulae in $\mathrm{ESO}^{4k-1}_{\mathrm{F}}$ without function variables. This then completes the proof.
\end{proof}


\section{Arbitrary Structures}

Based on the results presented in the previous two sections, we can then compare the expressive power of disjunctive and normal logic programs over arbitrary structures.


\begin{thm}\label{thm:arb_dlp2nlp}
$\mathrm{DLP}\simeq\mathrm{NLP}$ iff $\mathrm{DLP}\simeq_{\mathsf{FIN}}\mathrm{NLP}$.
\end{thm}

\begin{proof} (Sketch)
It is trivial from left to right. Now we show the converse direction. Assume $\mathrm{DLP}\simeq_{\mathsf{FIN}}\mathrm{NLP}$. Then, for each disjunctive logic program $\Pi$, there should be a normal logic program $\Pi^{\star}$ such that $\mathrm{SM}(\Pi)\equiv_{\mathsf{FIN}}\exists\sigma\mathrm{SM}(\Pi^{\star})$, where $\sigma$ is the set of all predicates and functions occurring in $\Pi^{\star}$ but not in $\Pi$. To show $\mathrm{DLP}\simeq\mathrm{NLP}$, our idea is to design a logic program testing whether or not the model currently considered is finite. If that is true, we then let $\Pi^{\star}$ work; otherwise, let $\Pi^{\diamond}$ which is developed for infinite structures work. To do this, we introduce two proposition constants, $\textsc{inf}$ and $\textsc{fin}$, as flags.
%
Let $\Pi_{\mathrm{inf}}$ be the program consisting of following rules:\vspace{-.1cm}
\begin{center}
\begin{tabular}{rrclrrcl}
\!\!1.\!\!\!\!\!&$\neg\textsc{arc}(x,y)$\!&\!\!\!\!$\rightarrow$\!\!\!\!&\!$\underline{\textsc{arc}}(x,y)$\hfill,\quad&\mbox{ }\!\!4.\!\!\!\!\!&$\neg \textsc{ok}_a(x)$\!&\!\!\!\!$\rightarrow$\!\!\!\!&\!$\textsc{ok}_a(x)$\hfill,\\
\!\!2.\!\!\!\!\!&$\neg\underline{\textsc{arc}}(x,y)$\!&\!\!\!\!$\rightarrow$\!\!\!\!&\!$\textsc{arc}(x,y)$\hfill,\quad&\mbox{ }\!\!5.\!\!\!\!\!&$\textsc{arc}(x,x)$\!&\!\!\!\!$\rightarrow$\!\!\!\!&\!$\underline{\textsc{inf}}$\hfill,\\
\!\!3.\!\!\!\!\!&$\textsc{arc}(x,y)$\!&\!\!\!\!$\rightarrow$\!\!\!\!&\!$\textsc{ok}_a(x)$\hfill,\quad&\mbox{ }\!\!6.\!\!\!\!\!&$\neg\underline{\textsc{inf}}$\!&\!\!\!\!$\rightarrow$\!\!\!\!&\!$\textsc{inf}$\hfill,\vspace{-.1cm}
\end{tabular}
\end{center}
and rule $\textsc{arc}(x,y)\wedge\textsc{arc}(y,z)\rightarrow\textsc{arc}(x,z)$. This program sets flag $\textsc{inf}$ to be true if the intended model is infinite.
When the intended model is finite, we use program $\Pi_{\mathrm{fin}}$ to set flag $\textsc{fin}$, which is obtained from $\pi_1$ in Lemma 6.4 of~\cite{EGM:97} by substituting $\textsc{fin}$ for $Order$ and by applying the shift operation in Section 4.5 of~\cite{BD:94}. Let $\Pi^{\diamond}_0$ ($\Pi^{\star}_0$) be the program obtained from $\Pi^{\diamond}$ ($\Pi^{\star}$, respectively) by adding $\textsc{inf}$ ($\textsc{fin}$, respectively) to the body of each rule as a conjunct. Let $\Pi^{\dag}$ be the union of $\Pi^{\diamond}_0$, $\Pi^{\star}_0$, $\Pi_{\mathrm{inf}}$ and $\Pi_{\mathrm{fin}}$. We can show
$
\mathrm{SM}(\Pi)\equiv\exists\sigma\mathrm{SM}(\Pi^{\dag})
$, where $\sigma$ is the set of all constants occurring in $\Pi^{\dag}$ but not in $\Pi$.
\end{proof}

\begin{rem}
In classical logic, it is well-known that separating languages over arbitrary structures is usually easier than that over finite structures. This is not surprise as arbitrary structures enjoy a lot of properties, including the compactness and the interpolation theorem, that fail on finite structures~\cite{EF:99}. In logic programming, it also seems that arbitrary structures are better-behaved than finite structures. For example, there are some preservation theorems that work on arbitrary structures, but not on finite structures~\cite{AG:94,ZZ:13}.
Therefore, the above result sheds a new insight on the stronger separations of $\mathrm{DLP}$ from $\mathrm{NLP}$ over finite structures.
\end{rem}

From Theorem \ref{thm:arb_dlp2nlp} and Proposition \ref{prop:fin_dlp2nlp}, we immediately have:

\begin{cor}
$\mathrm{DLP}\simeq\mathrm{NLP}$ iff $\mathrm{NP}=\mathrm{coNP}$.
\end{cor}

\section{Related Works and Conclusion}

Over Herbrand structures, \cite{EG:97,Schl:95} showed that both disjunctive and normal logic programs define the same class of database queries if functions are allowed. Our result over infinite structures is more general and stronger than theirs as Herbrand structures are only a special class of countable infinite structures. It is not clear whether or not their approach, which employs the inductive definability from~\cite{Barw:76,Mosc:74}, can be generalized to arbitrary infinite structures.

To the best of our knowledge, the weaker separation over finite structures in this paper gives us the first lower bound for arities of auxiliary predicates in the translatability from disjunctive logic programs into normal logic programs. Improving the lower bound will shed light on deeply understanding the expressive power of disjunctive and normal logic programs, which will be a challenging task in the further study. The equivalence of the translatability over finite structures and over arbitrary structures provides us a new perspective to achieve this goal. We will pursue this in the near future.



%
%

\bibliographystyle{named}
\bibliography{ijcai13}

\end{document}